\documentclass{colt2013}
\usepackage{amsfonts}
\usepackage{amsmath}
\usepackage{amssymb}
\usepackage{enumerate}
\usepackage{units}

\newenvironment{proofsketch}[1][Proof Sketch]{\begin{trivlist}
\item[\hskip \labelsep {\bfseries #1}]}{\end{trivlist}}

\newcommand{\commentout}[1]{}

\begin{document}

\title{
Horizon-Independent Optimal Prediction with Log-Loss in Exponential Families}
\coltauthor{\Name{Peter Bartlett} \Email{bartlett@cs.berkeley.edu}\\
\addr University of California at Berkeley, \\
Queensland University of Technology
\AND
\Name{Peter Gr\"{u}nwald} \Email{Peter.Grunwald@cwi.nl}\\
 \addr CWI  P.O. Box 94079
NL-1090
GB
Amsterdam
The Netherlands 
 \AND
 \Name{Peter Harremo\"{e}s} 
\Email{harremoes@ieee.org}\\
 \addr Copenhagen Business College
Denmark
\AND
\Name{Fares Hedayati} \Email{fareshed@eecs.berkeley.edu%
}\newline
\addr University of California at Berkeley
\AND
\Name{Wojciech Kot\l{}owski} \Email{wkotlowski@cs.put.poznan.pl}\\
\addr Pozna{\'n} University of Technology Poland
} \maketitle
\maketitle

\begin{abstract}

We study online learning under logarithmic loss with regular parametric
models. \newline
\cite{HedayatiB12} showed that a Bayesian prediction strategy with Jeffreys
prior and sequential normalized maximum likelihood (SNML) coincide and are
optimal if and only if the latter is exchangeable, and if and only if the
optimal strategy can be calculated without knowing the time horizon in
advance. They put forward the question what families have exchangeable SNML
strategies. This paper fully answers this open problem for one-dimensional
exponential families. The exchangeability can happen only for three classes
of natural exponential family distributions, namely the Gaussian, Gamma, and
the Tweedie exponential family of order $\nicefrac{3}{2}$.
\end{abstract}


\section{Introduction}

We work in the setting of online learning under logarithmic loss. Let $%
x_{1},x_{2},\ldots $, be a sequence of outcomes from $\mathcal{X}$ revealed
one at a time. We use $x^{t}$ to denote $\left( x_{1},x_{2},\cdots
,x_{t}\right) $, and $x_{m}^{n}$ to denote $\left( x_{m},x_{m+1},\cdots
,x_{n}\right) .$ At round $t$, after observing $x^{t-1}$, a forecaster
assigns a probability distribution on $\mathcal{X}$, denoted $p\left(
\,\cdot \mid x^{t-1}\right) $. Then, after $x_{t}$ is revealed, the
forecaster incurs the \emph{log loss} -$\mathrm{\ln }\;p\left( x_{t}\mid
x^{t-1}\right) .$ The performance of the strategy is measured relative to
the best in a reference set of strategies (experts). The difference between
the accumulated loss of the prediction strategy and the best expert in the
reference set is called the \emph{regret} \citep{CesaBianchiLugosi06}. The
goal is to minimize the regret in the worst case over all possible data
sequences.


In this paper our set of experts are i.i.d. exponential families of
distributions, examples of which include normal, Bernoulli, multinomial,
Gamma, Poisson, Pareto, geometric distributions and many others. If there is
a known time horizon $n$ of the game (sequence length), a well-known result
in the literature states that the minimax regret is achieved by the \emph{\
normalized maximum likelihood} strategy, NML for short %
\citep{Shtarkov87,Rissanen96}. If the parameter space of a $d$-dimensional
exponential family is constrained to a compact subset of the parameter
space, NML achieves regret $\frac{d}{2}\mathrm{\ln }\;n+O(1)$. For
unconstrained parameter spaces, the NML strategy is often not defined
because it relies on finiteness of the Shtarkov sum (or integral) and in
many application this sum is infinite. In these cases NML can be replaced by
the \emph{conditional normalized maximum likelihood strategy} (CNML), which
acts like NML, except that a small initial segment of the sequence is
observed before prediction starts and then the NML strategy is calculated
conditioned on that initial segment. Whereas NML is optimal in the sense of
achieving minimax regret (whenever it is finite), CNML is optimal in the
sense that it achieves minimax \emph{conditional\/} regret. Unfortunately
both CNML and (whenever it is defined) the original NML suffer from two
major drawbacks: the horizon $n$ of the problem needs to be known in
advance, and the strategy can be computationally expensive since it involves
marginalizing over all possible future subsequences up to iteration $n$.
These drawbacks motivated researchers to come up with an approximation to
CNML, known as \emph{sequential normalized maximum likelihood}, or SNML for
short \citep{TakimotoW00ALT,RissanenR07,RoosR08}.

SNML predictions coincide with those of the CNML distribution under the
assumption that the current iteration is the last iteration. Therefore, SNML
can be viewed as an approximation to CNML for which the time horizon of the
game does not need to be known. \citet{KotlowskiG11} showed that for general
exponential families SNML is optimal up to an $O\left( 1\right) $--term.
Interestingly, acting short-sighted and looking only one step ahead does not
hurt much.

A natural question to ask is if there are cases in which looking one step
ahead in the prediction game is \emph{exactly} the best one can do, even if
the time horizon is known? In other words, when do SNML and CNML coincide?
We believe answering this question is of fundamental importance for online
learning at least from the two following reasons. First, we know that in a
general sequential decision process, obtaining the optimal strategy requires
recursive solution of the Bellman equation by a backward induction. A
positive answer to the question above implies that we can avoid the backward
induction altogether, because the optimal strategy becomes time-horizon
independent: we get the same, optimal strategy no matter how far to the
future we look. Thus, we only need to analyze the worst case regret with
respect to the current outcome to be predicted.

Secondly, it has been shown \citep{KotlowskiG11,HedayatiB12AISTATS,HedayatiB12,Harremoes13}
that when CNML and SNML coincide, then they become Bayesian strategies and
the prior of the Bayesian strategy must be Jeffreys prior. In other words,
if CNML is time-horizon independent, then the Bayesian strategy with
Jeffreys prior is the (conditional) minimax strategy. %
\citet{HedayatiB12AISTATS,HedayatiB12} showed that this happens if and only
if SNML strategy is \emph{exchangeable}. Testing the exchangeability of the
sequential strategy is, however, hard, and does not lead to a simple
characterization of exponential families for which CNML=SNML holds.
Therefore, \citet{HedayatiB12} put forward the following question: in the
case of exponential families, what families have exchangeable sequential
normalized maximum likelihood strategies?

In this paper we give a complete answer to the CNML=SNML question, when the
reference set of experts is a single-parameter exponential family. We show
that there are essentially only three exponential families with time-horizon
independent minimax strategy (and hence optimal Bayesian strategy with
Jeffreys prior and optimal SNML). These families are Gamma, Gaussian, and
Tweedie $\nicefrac{3}{2}$ families (but also included are those families,
which can be obtained by a fixed transformation of variable from any of the
three above, e.g. Pareto, Laplace, Rayleigh and many others). This means
that only in these families, a Bayesian strategy with Jeffreys prior is
equivalent to SNML and both are equivalent to CNML and hence optimal in the
minimax conditional regret sense. More interestingly this implies that only
in these three families CNML becomes independent of the horizon, so that
one--stead ahead lookup becomes equivalent to $n$--step ahead lookup, where $%
n$ is the amount of data the player is eventually going to observe.

The paper is organized as follows. We introduce the mathematical context for
our results in Section~\ref{sec:notations}. We then give our main result in
Section~\ref{sec:main_results}, showing that Gamma, Gaussian and Tweedie %
\nicefrac{3}{2} family are the only families with time-horizon independent
minimax strategies. Short versions of the proofs are given in Section~\ref%
{sec:main_results}, and detailed proofs can be found in the appendix.%
We end with a short discussion in Section~\ref{sec:conclusion}.

\section{Set-Up}

\label{sec:notations} We work in the setup of \cite{HedayatiB12} and use
their definitions and notation except that we follow \cite{Grunwald07} in
the distinction between NML and CNML.

A sequential prediction strategy (or just `strategy' for short) $p$ is any
sequential probability assignment that, given a history $x^{t-1}$, defines $%
p\left( \cdot \left\vert x^{t-1}\right. \right) $, the conditional density
of $x_{t}\in \mathcal{X}$ with respect to a fixed underlying measure $%
\lambda $ on $\mathcal{X}$. As an example, we usually take $\lambda $ to be
the counting measure if $\mathcal{X}$ is discrete; and if $\mathcal{X}=%
\mathbb{R}^{d}$, $\lambda $ is taken to be Lebesgue measure.

A prediction strategy defines a joint distribution $p$ on sequences of
elements of $\mathcal{X}$ in the obvious way, 
\begin{equation*}
p\left( x^{n}\right) =\prod_{t=1}^{n}p\left( x_{t}\left\vert x^{t-1}\right.
\right) .
\end{equation*}%
Conversely, any probability distribution $p\left( x^{n}\right) $ on the set $%
\mathcal{X}^{n}$ defines a prediction strategy induced by its conditional
distributions $p\left( \cdot \left\vert x^{t-1}\right. \right) $ for $1\leq
t\leq n$ \citep{CesaBianchiLugosi06,Grunwald07}.

We try to come up with strategies which predict as well as the best element
of a reference set of 'experts', which we take to be $1$-dimensional \emph{%
i.i.d. natural exponential families}. For these families $\mathcal{X}$ can
be identified with a subset of $\mathbb{R}$ and the set of 'experts' is a
set of distributions $\left\{ p_{\theta }\mid \theta \in \Theta \right\} $
on $\mathbb{R}$, each of which is of the form 
\begin{equation}
p_{\theta }(x)=h(x)\mathrm{e}^{\theta x-A(\theta )},~\theta \in \Theta .
\label{exponential family}
\end{equation}%
Here $h$ is a reference measure, given as a density relative to the
underlying measure $\lambda $. Here $A$ is the cumulant generating function
given by $A(\theta )=\mathrm{\ln }\int ${\textrm{$e$}}${^{\theta x}\,\mathrm{%
d}h(x)}$. The so-called natural parameter space of the family is the set 
\begin{equation}
\Theta _{\text{full}}=\left\{ \theta \in \mathbb{R}\left\vert A(\theta
)<\infty \right. \right\}  \label{natural parameter space}
\end{equation}%
We will generally work with potentially restricted families with parameters
sets $\Theta $ that may be proper subsets of $\Theta _{\text{full}}$ \emph{%
and that we always require to have nonempty interior\/} (so for example, we
do not consider finite subfamilies). Families with $\Theta =\Theta _{\text{%
full}}$ are called \emph{full}.

The families are extended to $n$ outcomes by taking product distributions: $%
p_{\theta }\left( x^{n}\right) =\prod_{t=1}^{n}p_{\theta }\left(
x_{t}\right) $. In this way each member of the family defines a prediction
strategy $p_{\theta }$ such that $p_{\theta }\left( x_{t}\left\vert
x^{t-1}\right. \right) :=p_{\theta }\left( x_{t}\right) $ with $p_{\theta
}\left( x_{t}\right) $ given by (\ref{exponential family}). Note that we
never assume that data are i.i.d.; only the set of predictors we compare
ourselves to treat it as i.i.d.

According to the standard general definition of exponential families %
\citep{Barndorff78}, we can have $\theta f(x)$ instead of $\theta x$ in the
exponent of (\ref{exponential family}), for an arbitrary fixed function $f$.
Families with $f(x)=x$ are called \emph{natural\/} exponential families
relative to random vector $X$ (defined as $X(x)=x$).

However, as long as $f$ is smooth and 1-to-1, a general exponential family
with statistic $f(x)$ can always re-expressed as a natural exponential
family relative to a different random variable $Y=f\left( X\right) $ (i.e.
it defines exactly the same distributions on the underlying space), so our
restriction to natural families is actually quite mild; see also the
discussion right after our main result Theorem \ref{three_}.

Given a fixed horizon $n$ and a parameter space $\Theta $, the NML strategy %
\citep{Shtarkov87,Rissanen96} is defined via the joint probability
distribution 
\begin{equation}
p_{nml}^{(n)}\left( x^{n}\right) =\frac{\sup_{\theta \in \Theta
}\,\,p_{\theta }\left( x^{n}\right) }{\int_{\mathcal{X}^{n}}\sup_{\theta \in
\Theta }p_{\theta }\left( y^{n}\right) \,\mathrm{d}\lambda ^{n}\left(
y^{n}\right) },  \label{nml}
\end{equation}%
provided that the so-called \emph{Shtarkov\/ integral} in the denominator
exists. To ensure that the NML-distribution exists we will assume that the
parameter space is closed. For $t\leq n,$ the conditional probability
distribution is 
\begin{equation}
p_{nml}^{(n)}\left( x_{t}\left\vert x^{t-1}\right. \right) =\frac{%
p_{nml}^{(n)}\left( x^{t}\right) }{p_{nml}^{(n)}\left( x^{t-1}\right) }
\end{equation}%
where $p_{nml}^{(n)}\left( x^{t}\right) $ and $p_{nml}^{(n)}\left(
x^{t-1}\right) $ are marginalized joint probability distributions of $%
p_{nml}^{(n)}(x^{n})$: 
\begin{equation*}
p_{nml}^{(n)}\left( x^{t}\right) =\int_{\mathcal{X}^{n-t}}p_{nml}^{(n)}%
\left( x^{n}\right) \,\mathrm{d}\lambda ^{n-t}\left( x_{t+1}^{n}\right) .
\end{equation*}%
Note that the expression for the conditional distribution of a full-length
complement of a sequence $x_{t+1}^{n}$ given the initial part of the
sequence $x^{t}$ then simplifies to: 
\begin{equation}
p_{nml}^{(n)}\left( x_{t+1}^{n}\left\vert x^{t}\right. \right) =\frac{%
p_{nml}^{(n)}\left( x^{n}\right) }{p_{nml}^{(n)}\left( x^{t}\right) }=\frac{%
\sup_{\theta \in \Theta }p_{\theta }\left( x^{n}\right) }{\int_{\mathcal{X}%
^{n-t}}\sup_{\theta \in \Theta }p_{\theta }\left( x^{t}y^{n-t}\right) \,%
\mathrm{d}\lambda ^{n-t}\left( y^{n-t}\right) }.  \label{nmlcnml}
\end{equation}%
In many cases the NML strategy is undefined, due to the normalization factor
(Shtarkov integral) being infinite. In such cases, by conditioning on a
fixed initial sequence of length $m$ the problem usually goes away. The resulting 
\emph{\ conditional NML\/} (CNML) distribution achieves the minimax \emph{\
conditional\/} regret \cite[Chapter 11]{Grunwald07}. CNML is defined via the
conditional probability distribution in the following way 
\begin{equation}
p_{cnml}^{(n)}\left( x_{m+1}^{n}\left\vert x^{m}\right. \right) =\frac{%
\sup_{\theta \in \Theta }p_{\theta }\left( x^{n}\right) }{\int_{\mathcal{X}%
^{n-m}}\sup_{\theta \in \Theta }p_{\theta }\left( x^{m}y^{n-m}\right) \,%
\mathrm{d}\lambda ^{n-m}\left( y^{n-m}\right) }.  \label{cnml}
\end{equation}%
Note that (\ref{cnml}) coincides with (\ref{nmlcnml}), so CNML can be
considered a generalization of NML. NML and CNML are costly due to the
amount of marginalization at each round. Furthermore they are hor{izon\/}%
--dependent, i.e. the predictions to be made depend on the amount of data
that will eventually be seen. \cite{Grunwald07} discusses in detail why this
can be problematic. Two alternative strategies which avoid these issues are
the Bayesian strategies with Jeffreys prior and the sequential normalized
maximum likelihood strategy, SNML for short, as developed by \cite%
{RissanenR07,RoosR08}. SNML is defined via the conditional probability
distribution in the following way 
\begin{equation*}
p_{snml}\left( x_{t}\left\vert x^{t-1}\right. \right) =p_{cnml}^{(t)}\left(
x_{t}\left\vert x^{t-1}\right. \right) .
\end{equation*}%
\cite{KotlowskiG11} showed that SNML provides a reasonably good
approximation of CNML. At each point in time $t-1$, the SNML strategy for
predicting the next outcome $x_{t}$ may be viewed as the strategy that would
lead to minimax optimal conditional regret if the next step was the last
round of the game. Hence, it is essentially a \emph{last-step minimax\/}
strategy in the sense of \cite{TakimotoW00}.

The other alternative, the Bayesian strategy with Jeffreys prior, is also
defined via its conditional distributions as

\begin{equation*}
p_{\pi }\left( x_{t}\left\vert x^{t-1}\right. \right) \,\,=\int_{\theta \in
\Theta }p_{\theta }\left( x_{t}\right) \,\mathrm{d}\pi \left( \theta
\left\vert x^{t-1}\right. \right) .
\end{equation*}%
Here $\pi \left( \theta \left\vert x^{t-1}\right. \right) $ is the posterior
distribution based on prior $\pi (\cdot )$ and $\pi (\cdot )$ is \emph{%
Jeffreys prior} defined to be proportional to $I(\theta )^{\nicefrac{1}{2}}$
with $I$ being the Fisher information. A well-known result in the literature
says that if the parameter space is effectively smaller than the natural
parameter space then the Bayesian strategy with Jeffreys prior is
asymptotically minimax optimal (See chapters~7 and 8 in \cite{Grunwald07}).
The nice thing about these two alternatives is that unlike CNML they are
defined naturally via conditional probability distributions that are much
easier to compute. In general Jeffreys prior cannot be normalized (i.e. $%
\int I(\theta )^{\nicefrac{1}{2}}~\mathrm{d}\theta =\infty $) but for all
models used in applications its posterior after just one single observation
is proper (i.e. well-defined) and can be used for predictions; see below
Lemma~\ref{lem:dmaximal} for details (note though that there exist
pathological models where no finite number of observations will give a
proper Jeffreys posterior \citep{Harremoes13}).

\cite{HedayatiB12} proved that these two alternatives are exactly the same
as CNML and hence optimal if and only if SNML is exchangeable. Let $p$ be
any time horizon-independent sequential prediction strategy conditioned on
an initial sequence $x^{m}$, which for any $n>m$ and any $x_{m+1}^{n}$,
assigns a joint probability distribution $p\left( x_{m+1}^{n}\left\vert
x_{m}\right. \right) $. We say that $p$ is \emph{exchangeable} if for any $%
n>m$, any $x_{n}\in \mathcal{X}^{n}$, the joint probability $p\left(
x_{m+1}^{n}\left\vert x_{m}\right. \right) $ assigned to $x_{m+1}^{n}$ is
invariant under any permutation $\sigma $ on $\left\{ 1,\ldots ,n\right\} $
which leaves the initial part of data $x^{m}$ unchanged.

Thus, exchangeability of SNML means that the joint distribution of SNML
conditioned on initial data $x^{m}$ is invariant under any permutation of
the data sequence $x_{m+1}^n$. Exchangeability of SNML is usually hard to
check. The natural question to ask is whether there exists an equivalent
characterization that can be easily read off of the distribution or not? In
this paper we show that there are only three types of exponential family
distributions that have exchangeable SNML. For none of the these three
families the denominator in Equation~\ref{nml} is finite. Hence, for all
one-dimensional exponential families in which NML is defined it will be
horizon dependent and can neither agree with SNML nor with a Bayesian
strategy.

\section{Main Results}

\label{sec:main_results} We now provide a sequence of lemmas and theorems
that lead up to our main result, Theorem~\ref{three_}. We provide a full
proof of Lemma~\ref{lem:first} and the final Theorem~\ref{three_} in the
main text, since, while not at all the most difficult ones, these results
contain the key ideas for our reasoning. All other results are followed by a
short proof sketch/idea; full proofs of these results are in the appendix.
We first provide a number of definitions that will be used repeatedly.

\subsection{Definitions}

From now on, whenever we refer to an `exponential family', unless we
explicitly state otherwise, we mean a an i.i.d. natural $1$-dimensional
family as in (\ref{exponential family}).

Our analysis below involves various parameterizations of natural exponential
families, in particular the natural, the mean (see below) and the geodesic
(only used in the appendix) parameterization. We typically use $%
\Theta $ for (a subset of) the natural parameter space, $M$ for (a subset
of) the corresponding mean-value space and $B$ for the geodesic space, but
if statements hold for general diffeomorphic parameterizations we use $%
\Gamma $ to denote (subsets of) the parameter space (natural, mean and geodesic  parameterizations are all instances of `diffeomorphic'
parameterizations \citep[page 611]{Grunwald07}). We then denote parameters
by $\gamma $ and we let $\hat{\gamma}(x^{n})$ be the maximum likelihood
estimate for data $x^{n}$. If $x^{n}$ has no or several ML estimates, $\hat{%
\gamma}(x^{n})$ is undefined. We let $\hat{\Gamma}_{n}$ be the subset of ML
estimates for data of length $n$, i.e. the set of $\gamma \in \Gamma $ such
that $\gamma =\hat{\gamma}(x^{n})$ for some data $x^{n}$ of length $n$, and
we let $\hat{\Gamma}^{\circ }$ be the set of $\gamma $ in the \emph{%
interior\/} of $\Gamma $ that are contained in $\hat{\Gamma}_{n}$ for \emph{%
some $n$}. (recall that we always assume that $\Gamma $ is closed). We will
also used symbols $\hat{M}_{n},\hat{M}^{\circ },\hat{B}_{n},\hat{B}^{\circ
},\ldots $ to denote corresponding sets in particular parameterizations. $%
D\left( \left. \gamma _{0}\right\Vert \gamma _{1}\right) :=D\left( \left.
p_{\gamma _{0}}\right\Vert p_{\gamma _{1}}\right) $ denotes the KL
divergence of $\gamma _{1}$ to $\gamma _{0}$.

We recall the standard fact that every natural exponential family can be
parameterized by the mean value of $X$: for each $\theta $ in the natural
parameter space $\Theta $, we can define $\mu _{\theta }:=E_{p_{\theta }}%
\left[ X\right] $; then the mapping from $\theta $ to $\mu _{\theta }$ is
1-to-1 and strictly increasing, and the image $\mu (\Theta )$ is the
mean-value parameter space $M$. We use $\hat{\mu}\left( x^{n}\right) $ for
the maximum likelihood estimator in the mean-value parameter space. We will
frequently use the \emph{\ variance function\/} $V(\mu )$ which maps the
mean of the family to its variance, i.e. $V(\mu )$ is the variance of $%
p_{\mu }$. We note that the Fisher information $I(\mu )$ in the mean-value
parameterization is the inverse of $V\left( \mu \right) $ \cite[Chapter 18]%
{Grunwald07}. We also introduce the standard deviation $\sigma $ as a
function of the mean by $\sigma \left( \mu \right) =V\left( \mu \right) ^{%
\nicefrac{1}{2}}.$

\begin{definition}[convex core]
\label{convex core} Consider a natural exponential family as in (\ref%
{exponential family}). Let $x_{0}=\inf \left\{ x:x\in \mbox{support of }%
h\right\} $, and $x_{1}=\sup \left\{ x:x\in \mbox{support of }h\right\} $.
The {\textbf{convex core}} is the interval from $x_{0}$ to $x_{1}$ with $%
x_{0}$ included if and only if $h$ has a point mass in $x_{0}$, and with $%
x_{1}$ included if and only if $h$ has a point mass in $x_{1}$. We denote
this the convex core by {\textbf{cc}}.
\end{definition}

For example for a Bernoulli model, the convex core is $[0,1]$, with $0$ and $%
1$ included. The intuition is that the convex core includes mean--values
that can be achieved by distributions corresponding to natural parameter
values $\infty $ and/or -$\,\infty $, in the cases where these are
well--defined.

\begin{definition}[maximal]
An exponential family with {\textbf{maximal mean-value parameter space}} is
an exponential family where the mean value parameter space equals the convex
core $cc$.
\end{definition}

For example, truncated exponential families such as Bernoulli $[0.2,0.8]$ do
not satisfy the maximal mean-value condition.

\subsection{Lemmas that Abstractly Characterize SNML-Exchangeability}

We now present three lemmas, which give an abstract characterization of SNML
exchangeability. Then in Section~\ref{sec:main} we will make these concrete,
leading to our main theorem.

We let $m$ be the smallest $n$ such that for all $x^{n}\in \mathcal{\ X}^{n}$%
, $\int p_{\gamma }\left( x^{n}\right) I\left( \gamma \right) ^{%
\nicefrac{1}{2}}~\mathrm{d}\gamma <\infty $ and $\int_{\mathcal{X}%
^{k-n}}\sup_{\gamma \in \Gamma }p_{\gamma }\left( x^{n},y^{k-n}\right) ~%
\mathrm{d}\lambda ^{k-n}\left( y^{k-n}\right) <\infty $ for $k\geq n$, i.e.
such that Jeffreys' posterior 
\begin{equation*}
\pi \left( \gamma \left\vert x^{n}\right. \right) :=\frac{p_{\gamma }\left(
x^{n}\right) I(\gamma )^{\nicefrac{1}{2}}}{\int p_{\gamma }(x^{n})I\left(
\gamma \right) ^{\nicefrac{1}{2}}~\mathrm{d}\gamma }
\end{equation*}
is proper (integrates to 1) for any conditioning sequence of length equal to
or longer than $m$, and that the conditional minimax regret is finite. In
most applications $m=1.$ Note that this implies that CNML and SNML
conditioned on an initial sequence of length $m$ exist \citep{Harremoes13},
so that all three prediction strategies (Bayes with Jeffreys, CNML and SNML)
are well-defined. From now on, each time we mention CNML/SNML we mean
\textquotedblleft CNML/SNML conditioned on an initial sequence of length $m$%
\textquotedblright .

We call the distribution $p_{\gamma }$ \emph{regular\/} if, for all $x^{n}$
with $\hat{\gamma}\left( x^{n}\right) =\gamma $, 
\begin{equation*}
\mu _{\gamma }=\hat{\mu}\left( x^{n}\right) =E_{p_{\gamma }}[X]=n^{{\text{-}}%
1}\sum_{i=1}^{n}x_{i},
\end{equation*}
i.e. in the mean-value parameter space, the ML estimator is equal to the
observed average. This is always the case if the ML estimate is in the
interior of $\Gamma $ \cite[Chapter 18]{Grunwald07}, but if the ML estimate
is on the boundary there can be exceptions, e.g. if $\Gamma $ is a truncated
parameter set. The following lemma is central:

\begin{lemma}
\label{lem:first} Consider a natural exponential family as in (\ref%
{exponential family}) where the parameter set $\Gamma $ is an interval. If
the SNML distribution for such a family is exchangeable then for all $n>m$
there is a constant $C_{n}$ such that for all \emph{regular\/} $\gamma
_{0}\in \hat{\Gamma}_{n}$, we have: 
\begin{equation}
\int_{\Gamma }\mathrm{e}^{{\text{-}}nD\left( \left. \gamma _{0}\right\Vert
\gamma \right) }I(\gamma )^{\nicefrac{1}{2}}\,\mathrm{d}\gamma =C_{n}.
\label{rpa}
\end{equation}%
If furthermore the family has maximal mean-value parameter space, then the
SNML distribution for such a family is exchangeable \emph{if and only if\/}
for all $n>m$ there is a constant $C_{n}$ such that for all $\gamma _{0}\in 
\hat{\Gamma}_{n}$ 
\begin{equation}
\int_{\Gamma }\mathrm{e}^{{\text{-}}nD\left( \left. \gamma _{0}\right\Vert
\gamma \right) }I(\gamma )^{\nicefrac{1}{2}}\,\mathrm{d}\gamma =C_{n}.
\label{rp}
\end{equation}
\end{lemma}

The essence of Lemma \ref{lem:first} is that $C_n$ remains constant as $\gamma_0$
varies. This will be key to proving our main result.

\begin{proof}
\cite{HedayatiB12} showed that, if $\Gamma $ is an interval, then SNML
exchangeability is equivalent to that Bayes with Jeffreys prior and
CNML coincide. Thus, equivalently, we must have, for \emph{all\/} $%
x_{1},\cdots ,x_{n}\in \mathcal{X}^{n}$, and \emph{all\/} $t$, such that $%
n>t\geq m$, 
\begin{equation}
p_{\pi }\left( \left. x_{t+1}^{n}\right\vert x^{t}\right)
=p_{cnml}^{(n)}\left( \left. x_{t+1}^{n}\right\vert x^{t}\right) .
\label{eq:jeffreys_equal_cnml}
\end{equation}%
Since 
\begin{equation*}
p_{\pi }\left( \left. x_{t+1}^{n}\right\vert x^{t}\right) =\int_{\Gamma
}p_{\gamma }\left( x_{t+1}^{n}\right) \,\mathrm{d}\pi \left( \theta
\left\vert x^{t}\right. \right) =\int_{\Gamma }p_{\gamma }\left(
x_{t+1}^{n}\right) \frac{p_{\gamma }\left( x^{t}\right) I(\gamma )^{%
\nicefrac{1}{2}}}{\int_{\Gamma }p_{\gamma ^{\prime }}\left( x^{t}\right)
I\left( \gamma ^{\prime }\right) ^{\nicefrac{1}{2}}~\mathrm{d}\gamma
^{\prime }}\,\mathrm{d}\gamma ,
\end{equation*}%
and 
\begin{equation*}
p_{cnml}^{(n)}\left( \left. x_{t+1}^{n}\right\vert x^{t}\right) =\frac{p_{%
\hat{\gamma}\left( x^{n}\right) }\left( x^{n}\right) }{\int_{\mathcal{X}%
^{n-t}}p_{\hat{\gamma}(x^{t},y^{n-t})}\left( x^{t}y^{n-t}\right) \,\mathrm{d}%
\lambda ^{n-t}\left( y^{n-t}\right) }
\end{equation*}%
in the diffeomorphic parametrization $\Gamma $, (\ref{eq:jeffreys_equal_cnml}%
) is equivalent to 
\begin{equation}
\int_{\Gamma }p_{\gamma }\left( x^{n}\right) I(\gamma )^{\nicefrac{1}{2}}~%
\mathrm{d}\gamma =C\left( n,x^{t}\right) \times p_{\hat{\gamma}\left(
x^{n}\right) }\left( x^{n}\right) ,  \label{eq:equality_proof_lemma3}
\end{equation}%
where 
\begin{equation*}
C\left( n,x^{t}\right) =\frac{\int_{\Gamma }p_{\gamma ^{\prime }}\left(
x^{t}\right) I\left( \gamma ^{\prime }\right) ^{\nicefrac{1}{2}}~\mathrm{d}%
\gamma ^{\prime }}{\int_{\mathcal{X}^{n-t}}p_{\hat{\gamma}\left(
x^{t},y^{n-t}\right) }\left( x^{t}y^{n-t}\right) \,\mathrm{d}\lambda
^{n-t}\left( y^{n-t}\right) }.
\end{equation*}%
We now prove that $C\left( n,x^{t}\right) =C_{n}$, i.e. it may depend on $n$
but \emph{it does not depend on $x_{1},\ldots ,x_{n}$}. The key observation
is that (\ref{eq:equality_proof_lemma3}) is satisfied for any $t\geq m$, in
particular for $t=m$, so that $C\left( n,x^{t}\right) $ cannot depend on $%
x_{m+1}^{n}$. However, since $C\left( n,x^{t}\right) $ and all other terms
in (\ref{eq:equality_proof_lemma3}) are invariant under any permutation of $%
x^{t}$, we conclude that $C\left( n,x^{t}\right) $ does not depend on the
whole sequence $x^{n}$.

Now we divide both sides of (\ref{eq:equality_proof_lemma3}) by $p_{\hat{%
\gamma}(x^{n})}(x^{n})$ and we exponentiate inside the integral. This gives: 
\begin{equation}
\int_{\Gamma }\mathrm{e}^{{\text{-}}\ln \frac{p_{\hat{\gamma}(x^{n})}(x^{n})%
}{p_{\gamma }(x^{n})}}I(\gamma )^{\nicefrac{1}{2}}~\mathrm{d}\gamma =C_{n}.
\label{emp}
\end{equation}%
We have thus shown that, assuming $\Gamma $ is an interval, SNML
exchangeability is equivalent to the condition that (\ref{emp}) holds for a
fixed $C_{n}$, for all $x^{n}\in \mathcal{X}^{n}$.

Now for Part (1), Let $\gamma _{0}=\hat{\gamma}\left( x^{n}\right) $. We now
use the celebrated robustness property of exponential families 
\citep[Section 19.3,
Eq. 19.12]{Grunwald07}). 
This property says that for all $\gamma _{0}$ such that $p_{\gamma _{0}}$ is
regular, for all $x^{n}$ with $\hat{\gamma}(x^{n})=\gamma $, we have 
\begin{equation}
nD\left( \left. {\gamma }_{0}\right\Vert \gamma \right) =\ln \frac{p_{\hat{%
\gamma}(x^{n})}\left( x^{n}\right) }{p_{\gamma }\left( x^{n}\right) };
\label{eq:robust}
\end{equation}%
the result follows.

For Part (2), we note that, if the mean-value parameter space is maximal,
then it must be an interval, and all points in this space must be regular 
\cite[Section 19.3, Eq. 19.10]{Grunwald07}. The only-if direction follows
immediately by Part (1). To see the converse, we note that if the mean-value
parameter space is maximal, then the maximum likelihood estimator exists and
is unique for all $x^{n}\in \mathcal{X}^{n}$ (see \cite{Csiszar03}), and all 
$\gamma \in \Gamma $ are regular. Hence Equation (\ref{eq:robust}) holds for
all $x^{n}\in \mathcal{X}^{n}$ so that (\ref{rp}) implies that (\ref{emp})
holds for all $x^{n}\in \mathcal{X}^{n}$ and therefore that SNML
exchangeability holds.
\end{proof}

We will also need a second lemma relating SNML exchangeability to maximality:

\begin{lemma}
\label{lem:dmaximal} Consider a natural exponential family as in (\ref%
{exponential family}). If the family is SNML exchangeable, then the
mean-value parameter space is maximal.
\end{lemma}

\begin{proofsketch}
In our definition of exponential families we require that the parameter set $%
\Gamma $ has non-empty interior, thus we may assume that it contains an
interval. We can then show by approximating the integral in (\ref{rpa}) by a
Gaussian integral using standard Laplace-approximation techniques (as in
e.g. \cite[Chapter 7]{Grunwald07}) that, for general 1-dimensional
exponential families, the integral in (\ref{rpa}) converges to $\left( 
\nicefrac{2
\pi}{n}\right) ^{\nicefrac{1}{2}}$ for any $\gamma _{0}$ in the interior of $%
\Gamma $. If SNML exchangeability holds, then we can show using Lemma~\ref%
{lem:first} and continuity that this must also hold for all boundary points
of $\Gamma $. But if the parameter space is not maximal, then the same
standard Laplace approximation of the integral in (\ref{rpa}) gives that the
integral converges to $\left( \nicefrac{1}{2}\right) \left( 
\nicefrac{2
\pi}{n}\right) ^{\nicefrac{1}{2}}$ and we have a contradiction.
\end{proofsketch}

\subsection{Preparing and Stating the Main Theorem}

\label{sec:main} In the following we will use the \emph{Tweedie exponential
families} of order $\nicefrac{3}{2}$ \cite{Jorgensen1997}. These are natural
exponential families characterized by a variance function of the form $%
V\left( \mu \right) =k\mu ^{\nicefrac{3}{2}}$, where $\mu $ is the mean and $%
V\left( \mu \right) $ is the variance function defined earlier (i.e. $V(\mu )
$ is the variance of $p_{\mu }$). Each of the elements is a compound Poisson
distribution. It is obtained as follows. Let $X_{i}$ denote i.i.d.
exponentially distributed random variables with mean $\nu $ and let $N$
denote an independent Poisson distributed random variable with mean  $\nu k^{%
{\text{-}}2}.$ Then the elements in the exponential family are distributions
of random variables of the form 
\begin{equation*}
Z=\sum_{i=1}^{N}X_{i}
\end{equation*}%
for different values of the parameter $\nu $ \citep{Jorgensen1997}. It is
interesting to note that such distributions have a point mass in 0 so that
the left tail gives a finite contribution to the Shtarkov integral but the
right tail is light and gives an infinite contribution to the Shtarkov
integral. Hence this family does not have finite minimax regret.

\begin{lemma}
\label{lem2} The following three types of exponential families are SNML
exchangeable: The full Gaussian location families with fixed $\sigma ^{2}>0$%
, the full Gamma distributions with shape parameter $k>0,$ and the full
Tweedie family of order $\nicefrac{3}{2}$.
\end{lemma}

\begin{proofsketch}
It is straightforward to check that all three families have maximal
mean-value parameter space. The result now follows by checking that (\ref{rp}%
) holds for these families, which is relatively straightforward by taking
derivatives of the cumulant generating function.
\end{proofsketch}

\begin{remark}
\label{rem:nakagami} \emph{\textrm{What we call \textquotedblleft Gamma"
here includes also Pareto, Laplace, Rayleigh, Levy, Nakagami and many other
families of distribution that are derived from the Gamma family by a smooth
one--to--one transformation. As the next lemma shows smooth one--to--one
transformations preserve SNML exchangeability.}}
\end{remark}

\begin{lemma}
\label{one--to--one transformation} Suppose $\left\{ \left. p_{\gamma
}(\,\cdot \,)\right\vert \gamma \in \Gamma \right\} $ indexes an exponential
family for a r.v. $X$ that is SNML exchangeable. Let $Y=f\left( X\right) $
for some smooth one--to--one function $f$ and let $q_{\gamma }(\,\cdot \,)$
be the density of $Y$ under $p_{\gamma }(\,\cdot \,)$. Then the family $%
\left\{ \left. q_{\gamma }\left( \cdot \right) \right\vert \gamma \in \Gamma
\right\} $ is SNML exchangeable as well.
\end{lemma}

\begin{proofsketch}
This is (almost) immediate from the definition of exchangeability.
\end{proofsketch}

\begin{example}As an example consider a random variable $X$ with a Gamma
distribution of the form $Gamma\left( \nicefrac{1}{2},\nicefrac{c}{2}\right) 
$ with density%
\begin{equation*}
\left( \frac{c}{2\pi}\right) ^{\nicefrac{1}{2}}x^{{\text{-}}%
\nicefrac{1}{2}}\mathrm{e}^{{\text{-}}\frac{xc}{2}}.
\end{equation*}%
Now if $X$ goes through the one--to--one transformation $f\left( X\right) =%
\nicefrac{1}{X}$ then 
\begin{equation*}
\frac{1}{X}\,\sim \,Inverse\mbox{-}Gamma\left( %
\nicefrac{1}{2},\nicefrac{c}{2}\right) 
\end{equation*}
with density 
\begin{equation*}
\left( \frac{c}{2\pi}\right) ^{\nicefrac{1}{2}}x^{{\text{-}}%
\nicefrac{3}{2}}\mathrm{e}^{{\text{-}}\frac{c}{2x}}.
\end{equation*}%
This is the same as the density of a $Levy\left( 0,c\right) .$ Hence $%
Levy\left( 0,c\right) $ is also SNML exchangeable. It is indeed easy to
directly verify the SNML exchangeability of $Levy\left( 0,c\right) $ using
Lemma~\ref{lem:first}.\end{example}

\begin{theorem}
\label{disjunction of conditions} Consider a natural exponential family as
in (\ref{exponential family}). A
necessary condition for SNML exchangeability is that the standard deviation
as function of the mean satisfies the differential equation 
\begin{equation}
\left( \frac{\mathrm{d}\sigma }{\mathrm{d}\mu }\right) ^{2}+3\sigma \frac{%
\mathrm{d}^{2}\sigma }{\mathrm{d}\mu ^{2}}=\mathrm{const}(\mu ).
\label{eq:sigmadiff}
\end{equation}
\end{theorem}

\begin{proofsketch}
By Lemma~\ref{lem:dmaximal} we may assume that the family has maximal
mean-value parameter space. A fifth-order (!) Taylor expansion of
(\ref{rp}) rewritten in the geodesic parameterization (see
(\ref{eq:geodesicform}) in the appendix) gives terms of
different order in $n,$ and each term should be constant. Equation \ref%
{eq:sigmadiff} corresponds to the first non-trivial term in the expansion.
\end{proofsketch}

\begin{theorem}
\label{sufficiency} Consider a natural exponential family as in (\ref%
{exponential family}) with {maximal mean-value parameter space}. A necessary
condition for SNML exchangeability is that the variance function is given by%
\begin{equation}
V\left( \mu \right) =\left( k\mu +\ell \right) ^{2}  \label{eq:varians1}
\end{equation}%
or%
\begin{equation}
V\left( \mu \right) =\left( k\mu +\ell \right) ^{\nicefrac{3}{2}}
\label{eq:varians2}
\end{equation}%
for some constants $k$ and $\ell .$
\end{theorem}

\begin{proofsketch}
The differential equation (\ref{eq:sigmadiff}) can be rephrased in terms of
the variance function. Two solutions are (\ref{eq:varians1}) or (\ref%
{eq:varians2}). Other potential solutions are ruled out by a higher-order
(in fact 7th-order!!) expansion.
\end{proofsketch}

Now we are ready to state our main theorem. We need one more definition: we
say that a full exponential family of form (\ref{exponential family}) is a 
\emph{linear transformation\/} of another full exponential family if, for
some fixed $a,b$, it is the set of distributions given by (\ref{exponential
family}), with each occurrence of $x$ replaced by $ax+b$.

\begin{remark}
\emph{By Remark~\ref{rem:nakagami}, linear transformations preserve SNML
exchangeability. In a Gaussian location family translating by $b$ replaces a
distribution by another distribution of the same exponential family and the
Gaussian location families are the only families with this property. Scaling
of a Gamma distribution by positive $a$ gives another Gamma distribution in
the same exponential family and the Gamma families are the only exponential
families with this property.}
\end{remark}

\begin{theorem}
The only natural 1-dimensional i.i.d. exponential families that have
exchangeable SNML are the following three: \label{three_}

\begin{itemize}
\item The full Gaussian location families with arbitrary but fixed $%
\sigma^{2}>0$.

\item The full Gamma exponential family with fixed shape parameter and
linear transformations of it.

\item The full Tweedie exponential family of order $\nicefrac{3}{2}$ and
linear transformations of it.
\end{itemize}
\end{theorem}

Before we prove this theorem, let us briefly discuss its generality.

As we already indicated below (\ref{exponential family}), every exponential
family defined with respect to a sufficient statistic $f(X)$ can be
re-expressed as a natural family with respect to $X$ as long as $f$ is
smooth and 1-to-1. Thus the theorem also determines SNML exchangeability for
general 1-dimensional i.i.d. exponential families with such $f.$ Namely, if
such a family, when mapped to a natural family, becomes the Gamma, Gaussian
or Tweedie $\nicefrac{3}{2}$ family, then it is SNML exchangeable; otherwise
it is not. The former is the case, for, for example, the Pareto and other
families mentioned in Remark~\ref{rem:nakagami}; the latter is the case,
for, for example, the Bernoulli and Poisson distributions.

\begin{proof}
Lemma~\ref{lem2} says that these three families are SNML exchangeable. As we
know that SNML exchangeability can only happen for families with maximal
mean-value parameter space (Lemma~\ref{lem:dmaximal}), we focus on these
families only. 
Thus, it is left to show that no other family with maximal mean-value
parameter space is SNML exchangeable. 

Theorem~\ref{sufficiency} gives the necessary condition for SNML
exchangeability in terms of the variance function. Now we look at each case
separately. The first part of the disjunction is the Equation \ref%
{eq:varians1}, where the variance function is quadratic. Exponential
families with quadratic variance functions have been classified by \cite%
{Morris82}. His result is that modulo linear transformations the only
exponential families with quadratic variance functions are Gaussian,
Poisson, Gamma, binomial, negative binomial, and the exotic hyperbolic secant
distribution. Of these only the Gaussian and the Gamma families have the
desired form. We note that the exponential distributions are special cases
of Gamma distributions.

Now we get to the second case where the variance function is given by
Equation \ref{eq:varians2}. If $c_{1}=0$ we get an exponential family where
the variance is constant, i.e. the family is the Gaussian translation
family. Then the term $k$ corresponds to a translation of the exponential
family and we may assume that $k=0.$ If $c_{1}\neq 0$ we can scale up or
down and obtain the equation 
\begin{equation}  \label{eq:beq}
V(\mu) =2\mu^{\nicefrac{3}{2}}.
\end{equation}%
There exists an exponential family with this variance function, namely the
Tweedie family of order $\nicefrac{3}{2}$ with $V\left( \mu \right) =2\mu ^{%
\nicefrac{3}{2}}$. Since exponential families are uniquely determined by
their variance function \cite{Morris82}, the Tweedie family of order $\nicefrac{3}{2}$
is the only family satisfying (\ref{eq:beq}).
\end{proof}

\section{Discussion}

\label{sec:conclusion}

The present paper has focused on 1-dimensional exponential families with
non-empty interior parameter spaces. Any model that admits a 1-dimensional
sufficient statistic can be embedded in a one dimensional exponential
family. One can prove that SNML exchangeability implies that the parameter
space must have non-empty interior, thus strengthening our results further,
but the limited space did not allow us to go into this problem here.

We do not have any general results for the multidimensional case, but we can
make a few observations: products of models that are SNML exchangeable are
also exchangeable. All multidimensional Gaussian location models can be
obtained in this way by a suitable choice of coordinate system. The only
other SNML exchangeable models we know of in higher dimensions are Gaussian
models where the mean is unknown and the scaling of the covariance matrix is
unknown. This can be seen from the fact that a sum of squared Gaussian
variables has a Gamma distribution. The Tweedie family of order $%
\nicefrac{3}{2}$ does not seem to play any interesting role in higher
dimensions, because it cannot be combined with the other distributions.

One of the consequences of this paper is that for 1-dimensional
exponential families, NML (if it is defined without conditioning) will
always be horizon dependent. We conjecture that this conclusion will
hold for arbitrary models. Only conditional versions of NML allow the
kind of consistency that we call SNML exchangeability, and even after
conditioning, SNML exchangeability is restricted to a few but very
important models.

\section*{Acknowledgements}
Wojciech Kot{\l}owski has been supported by the Foundation of Polish Science under the Homing Plus programme. We gratefully acknowledge the support of the NSF through grant
CCF-1115788 and of the Australian Research Council through Australian
Laureate Fellowship FL110100281.

\section*{Appendix}

In the proofs we introduce $\tau $ as short for $2\pi .$

\begin{proof}[Proof of Lemma \protect\ref{lem:dmaximal}]
Without loss of generality consider the mean-value parameter space. Assume
that the given exponential family is SNML-exchangeable and, without loss of
generality, that the parameter space contains an interval $\left[ \mu
_{0},\mu _{1}\right] $ with $\mu _{0}<\mu _{1}$. By Lemma~\ref{lem:first} we
have for all $n$, and all regular points in $x\in \hat{M}_{n}\cap \left[ \mu
_{0},\mu _{1}\right] $ that 
\begin{equation}
\int_{\lbrack \mu _{0},\mu _{1}]}\frac{\mathrm{e}^{{\text{-}}nD\left( x\Vert
\mu \right) }}{\sigma \left( \mu \right) }~\mathrm{d}\mu =C_{n}  \label{1}
\end{equation}%
is independent of $x$. Note that all points in the interior of $\left[ \mu
_{0},\mu _{1}\right] $ must be regular \cite[Section 19.3, Eq. 19.10]%
{Grunwald07}.

By a standard Laplace approximation of the integral in (\ref{1}) (done by a
Taylor approximation of the KL divergence, $D\left( x\Vert \mu \right)
\approx \frac{1}{2}\left( x-\mu \right) ^{2}V\left( x\right) ^{{\text{-}}1}$%
, so that for large $n$ the integral becomes approximately Gaussian) we get,
for each closed interval $M_{c}$ that is a subset of the convex core, that for
each $x$ in the interior of $M_{c}$, 
\begin{equation}
\frac{\int_{M_{c}}\frac{\mathrm{e}^{{\text{-}}nD\left( x\Vert \mu \right) }}{%
\sigma \left( \mu \right) }~\mathrm{d}\mu }{\left( \frac{\tau }{n}\right) ^{%
\nicefrac{1}{2}}}\rightarrow 1  \label{2}
\end{equation}%
and 
\begin{equation}
\frac{\int_{\{\mu \in M_{c}:\mu \geq x\}}\frac{\mathrm{e}^{{\text{-}}%
nD\left( x\Vert \mu \right) }}{\sigma \left( \mu \right) }~\mathrm{d}\mu }{%
\left( \frac{\tau }{n}\right) ^{\nicefrac{1}{2}}}\rightarrow \frac{1}{2}
\label{2bis}.
\end{equation}

For a precise statement and proof of these results, see e.g. \cite[Theorem
8.1 combined with Eq. (8.14)]{Grunwald07}. Combining (\ref{2}) with (\ref{1}%
), taking $M_{c}=\left[ \mu _{0},\mu _{1}\right] $, it follows that $%
C_{n}\rightarrow (\nicefrac{\tau}{n})^{\nicefrac{1}{2}}$. Now for each $%
\epsilon >0$ there is an $n$ such that $x\in \hat{M}_{n}\cap M$ and $|x-\mu
_{0}|<\epsilon $. Hence by continuity the equality (\ref{1}) also holds for $%
x=\mu _{0}$, so we get 
\begin{equation}
\frac{\int_{\lbrack \mu _{0},\mu _{1}]}\frac{\mathrm{e}^{{\text{-}}nD\left(
\mu _{0}\Vert \mu \right) }}{\sigma \left( \mu \right) }~\mathrm{d}\mu }{%
\left( \frac{\tau }{n}\right) ^{\nicefrac{1}{2}}}\rightarrow 1  \label{3}.
\end{equation}%
Now assume by means of contradiction that the convex core $cc$ includes an $%
x^{\prime }<\mu _{0}$ with $x^{\prime }\not\in M$ ($M$ being the parameter
space of the family), and let $M^{\prime }=[x^{\prime },\mu _{1}]$. Then $%
\mu _{0}$ is in the interior of $M^{\prime }$ and so, taking $%
M_{c}=M^{\prime }$, (\ref{2bis}) with $x=\mu _{0}$ gives that the same
integral as in (\ref{3}) converges to $\nicefrac{1}{2}$; we have arrived at
a contradiction.

In the same way, one proves that there can be no $x^{\prime }> \mu_1$ with $%
x^{\prime }$ in the convex core. Thus, the interval must coincide with the
convex core, which is what we had to prove.
\end{proof}

\begin{proof}[Lemma~\protect\ref{lem2}]
For each of the families it is sufficient to prove that 
\begin{equation*}
\int_{cc}\frac{\mathrm{e}^{{\text{-}}nD\left( \left. \gamma _{0}\right\Vert {%
\gamma }\right) }}{\sigma \left( \gamma \right) }~\mathrm{d}\gamma 
\end{equation*}%
does not depend on $\gamma _{0}\in cc$ where $cc$ denotes the convex core of
the family.

In the Gaussian location family with variance $\sigma ^{2}$ we have $D\left(
\left. {\gamma }_{0}\right\Vert {\gamma }\right) =D\left( {0}\left\Vert {%
\gamma -\gamma }_{0}\right. \right) $, and $V\left( \gamma \right) =\sigma
^{2},$ so the integral is invariant because of the invariance of the
Lebesgue integral.

The scaling property of the Gamma families imply that $D\left( \left. {%
\gamma }_{0}\right\Vert {\gamma }\right) =D\left( {1}\left\Vert {\nicefrac{
\gamma}{ \gamma _{0}}}\right. \right) .$ For the Gamma family with shape
parameter $k$ we have $V\left( \gamma \right) =\gamma ^{2}/k.$ Hence the
integral equals 
\begin{eqnarray*}
\int_{0}^{\infty }\frac{\mathrm{e}^{{\text{-}}nD\left( \left. {\gamma }%
_{0}\right\Vert {\gamma }\right) }}{\left( \gamma ^{2}/k\right) ^{%
\nicefrac{1}{2}}}\,\text{d}\gamma  &=&k^{\nicefrac{1}{2}}\int_{0}^{\infty }%
\frac{\mathrm{e}^{{\text{-}}nD\left( {1}\left\Vert \nicefrac{\gamma}{\gamma
_{0}}\right. \right) }}{\gamma }\,\text{d}\gamma  \\
&=&k^{\nicefrac{1}{2}}\int_{0}^{\infty }\frac{\mathrm{e}^{{\text{-}}nD\left( 
{1}\Vert {t}\right) }}{t}\,\text{d}t,
\end{eqnarray*}%
where we have used the substitution $t=\nicefrac{\gamma}{\gamma_0}$. Hence
the integral does not depend on $\gamma _{0}.$

We consider the Tweedie family of order $\nicefrac{3}{2}$. Then the
divergence can be calculated as 
\begin{eqnarray*}
D\left( \left. \mu _{0}\right\Vert \mu _{1}\right)  &=&\int_{\mu _{0}}^{\mu
_{1}}\frac{\mu -\mu _{0}}{2\mu ^{\nicefrac{3}{2}}}~\mathrm{d}\mu  \\
&=&\left[ \mu ^{\nicefrac{1}{2}}+\mu _{0}\mu ^{{\text{-}}\nicefrac{1}{2}}%
\right] {}_{\mu _{0}}^{\mu _{1}} \\
&=&\mu _{1}^{\nicefrac{1}{2}}+\mu _{0}\mu _{1}^{{\text{-}}\nicefrac{1}{2}%
}-2\mu _{0}^{\nicefrac{1}{2}} \\
&=&\frac{\left( \mu _{1}^{\nicefrac{1}{2}}-\mu _{0}^{\nicefrac{1}{2}}\right)
^{2}}{\mu _{1}^{\nicefrac{1}{2}}}.
\end{eqnarray*}%
Therefore we have to prove that the following integral is constant 
\begin{eqnarray*}
\int_{0}^{\infty }\exp \left( {\text{-}}nD\left( \gamma _{0}\Vert \gamma
\right) \right) \sigma \left( \gamma \right) ^{{\text{-}}1}~\mathrm{d}\gamma
&=&\int_{0}^{\infty }\exp \left( {\text{-}}n\frac{\left( \gamma ^{%
\nicefrac{1}{2}}-\gamma _{0}^{\nicefrac{1}{2}}\right) ^{2}}{\gamma ^{%
\nicefrac{1}{2}}}\right) \gamma ^{{\text{-}}\nicefrac{3}{4}}~\mathrm{d}%
\gamma  \\
&=&\int_{0}^{\infty }\exp \left( {\text{-}}\frac{\left( n\gamma ^{%
\nicefrac{1}{2}}-n\gamma _{0}^{\nicefrac{1}{2}}\right) ^{2}}{n\gamma ^{%
\nicefrac{1}{2}}}\right) \gamma ^{\textrm{-}\nicefrac{3}{4}}~\mathrm{d}\gamma .
\end{eqnarray*}%
The substitution $\gamma =t^{4}n^{{\text{-}}2}$ gives 
\begin{equation*}
\frac{4}{n^{\nicefrac{1}{2}}}\int_{0}^{\infty }\exp \left( {\text{-}}\frac{%
\left( t^{2}-n\gamma _{0}^{\nicefrac{1}{2}}\right) ^{2}}{t^{2}}\right) ~%
\mathrm{d}t.
\end{equation*}%
This integral is independent of $\gamma _{0},$ which proves the theorem.
\end{proof}

\begin{proof}[Lemma~\protect\ref{one--to--one transformation}]
Since the family $p_{\gamma }(\,\cdot \,)$ is SNML exchangeable, for
any $n>m$ the following joint distribution is invariant under permutations
of $x^{n}$ that leaves $x^{m}$ invariant: 
\begin{equation}
p_{snml}\left( \left. x_{m+1}^{n}\right\vert x^{m}\right) =\prod_{t=m+1}^{n}%
\frac{\sup_{\gamma }p_{\gamma }(x^{t})}{\int_{\mathcal{X}}\sup_{\gamma
}p_{\gamma }(x^{t-1},x)~\mathrm{d}x}
\end{equation}%
Now under the $Y=f(X)$ transformation the density of $Y$ becomes 
\begin{equation}
q_{\gamma }(y)=p_{\gamma }(f^{{\text{-}}1}(y))\left\vert \frac{\mathrm{d}%
\,f^{-1}(y)}{\mathrm{d}\,y}\right\vert ~.
\end{equation}%
For the ease of notation we let $v(y)=\left\vert \frac{\mathrm{d}\,f^{{\text{%
-}}1}(y)}{\mathrm{d}\,y}\right\vert ~$. Hence $q_{\gamma }(y)=p_{\gamma
}\left( f^{{\text{-}}1}\left( y\right) \right) v(y)$ and 
\begin{eqnarray*}
p_{snml}\left( \left. y_{m+1}^{n}\right\vert y^{m}\right) 
&=&\prod_{t=m+1}^{n}\frac{\sup_{\gamma }q_{\gamma }(y^{t})}{\int_{\mathcal{X}%
}\sup_{\gamma }q_{\gamma }\left( y^{t-1},y\right) ~\mathrm{d}y} \\
&=&\prod_{t=m+1}^{n}\frac{\sup_{\gamma }q_{\gamma }\left( f(x_{1})\cdots
f(x_{t})\right) }{\int_{\mathcal{X}}\sup_{\gamma }q_{\gamma }\left(
f(x_{1})\cdots f(x_{t-1}),y\right) ~\mathrm{d}y} \\
&=&\prod_{t=m+1}^{n}\frac{\sup_{\gamma }p_{\gamma }\left( x_{1}\cdots
x_{t}\right) \prod_{j=1}^{t}v(y_{j})}{\int_{\mathcal{X}}\sup_{\gamma
}p_{\gamma }\left( x_{1}\cdots x_{t-1},f^{-1}(y)\right)
\prod_{j=1}^{t-1}v(y_{j})v(y)~\mathrm{d}y} \\
&=&\prod_{t=m+1}^{n}\frac{\sup_{\gamma }p_{\gamma }(x^{t})v(y_{t})}{\int_{%
\mathcal{X}}\sup_{\gamma }p_{\gamma }(x^{t-1},f^{-1}(y))v(y)~\mathrm{d}y} \\
&=&\prod_{t=m+1}^{n}\frac{\sup_{\gamma }p_{\theta }(x^{t})v(y_{t})}{\int_{%
\mathcal{X}}\sup_{\gamma }p_{\gamma }(x^{t-1},x)~\mathrm{d}x} \\
&=&p_{snml}\left( \left. x_{m+1}^{n}\right\vert x^{m}\right)
\prod_{t=m+1}^{n}v(y_{t})~.
\end{eqnarray*}%
Hence $p_{snml}\left( \left. y_{m+1}^{n}\right\vert y^{m}\right) $ too is
invariant under any permutation of $y^{n}$ leaving $y^{m}$ invariant, and
hence exchangeable. Note that in the last but one equation we used the
change of variable $f^{{\text{-}}1}(y)=x$ and the fact that $v(y)\mathrm{d}y=%
\mathrm{d}x~.$
\end{proof}

Now we are ready to state the next theorem which is simply a disjunction of
two conditions necessary for SNML exchangeability in a parameterization
called geodesic. The \emph{geodesic parameterization\/} is the
parameterization in which the Fisher information is constant. We will denote
parameters in this parameterization by $\beta $ with parameter set $B$. We
can reparameterize from the natural parameter space $\Theta _{\text{full}}$
to the geodesic space by setting: 
\begin{eqnarray}
\beta &=&\int_{\theta _{0}}^{\theta }I(s)^{\nicefrac{1}{2}}\,\mathrm{d}s
\label{geo} \\
&=&\int_{\mu _{0}}^{\mu }\frac{1}{\sigma \left( t\right) }~\mathrm{d}t,
\end{eqnarray}%
so that $\mathrm{d}\beta =I(\theta )^{\nicefrac{1}{2}}~\mathrm{d}\theta =%
\mathrm{d}\mu /\sigma \left( \mu \right) $. Note that this is a bijection.
This allows us to replace the integration measure in the condition of Lemma~%
\ref{lem:first} and we get a condition equivalent to (\ref{rp}): for any $n>m$ the
following is independent of $\beta _{0}\in \hat{B}^{n}$ 
\begin{equation}
\int_{B}\mathrm{e}^{{\text{-}}nD\left( \left. \beta _{0}\right\Vert \beta
\right) }\,\mathrm{d}\beta .  \label{eq:geodesicform}
\end{equation}

\begin{proof}[Theorem~\protect\ref{disjunction of conditions}]
We denote the integral in Equation \ref{eq:geodesicform} by $s(\beta _{0},n)$.
We may assume the family has maximal mean-value parameter space, so
that (\ref{eq:geodesicform}) must hold for all 
$\beta_0 \in \hat{B}_n$, all $n$.
First we will establish the following relation between the geodesic
parametrization and the mean value parameterization 
\begin{eqnarray*}
\frac{\partial }{\partial \beta }(\ldots ) &=&\frac{\mathrm{d}\mu }{\mathrm{d%
}\beta }\frac{\partial }{\partial \mu }(\ldots ) \\
&=&\sigma \left( \mu \right) \frac{\partial }{\partial \mu }(\ldots ),
\end{eqnarray*}%
because $\frac{\mathrm{d}\beta }{\mathrm{d}\mu }=\sigma^{{\text{-}}1}
\left( \mu \right) $. We use the fact that $D\left( \left. \beta
_{0}\right\Vert \beta \right) =D\left( \left. \mu _{0}\right\Vert \mu
\right) $, where $\mu =\mu \left( \beta \right) $ and $\mu _{0}=\mu \left(
\beta _{0}\right) $ are corresponding parameters in different
parametrizations. 
\begin{align}
D\left( \left. \beta _{0}\right\Vert \beta \right) & =\mu _{0}\cdot \left(
\theta _{0}-\theta \right) +A\left( \theta \right) -A\left( \theta
_{0}\right)  \notag \\
\frac{\partial D\left( \left. \beta _{0}\right\Vert \beta \right) }{\partial
\beta }& =\left( \mu -\mu _{0}\right) \cdot \sigma ^{{\text{-}}1},  \notag \\
\frac{\partial ^{2}D\left( \left. \beta _{0}\right\Vert \beta \right) }{%
\partial \beta ^{2}}& =1-\left( \mu -\mu _{0}\right) \cdot \sigma ^{{\text{-}%
}1}\frac{\mathrm{d}\sigma }{\mathrm{d}\mu }.  \label{eq:andenafledt}
\end{align}%
Hence $D_2 =1$, where $D_n$ denotes $\frac{\partial^n D\left( \left. \beta _{0}\right\Vert \beta \right) }{\partial
\beta^n } \Big|_{\beta = \beta_0}$ throughout this section. 

A Taylor expansion of Equation \ref{eq:geodesicform} as function of $n$ gives that
certain Taylor coefficients must equal zero and an elaborate calculation of
the Taylor coefficient leads to Equation \ref{eq:sigmadiff}.

Using a fifth-order Taylor expansion we will show the following: 
\begin{equation}
s(\beta _{0},n)=\Phi +n^{{\text{-}}\nicefrac{3}{2}}\cdot 3\tau ^{%
\nicefrac{1}{2}}\cdot u\left( \beta _{0}\right) +O\left( n^{{\text{-}}%
2}\right)   \label{s}
\end{equation}%
where 
\begin{equation}
u\left( \beta _{0}\right) =\frac{5}{2}\cdot \left( \frac{D_{3}}{3!}\right)
^{2}-\frac{D_{4}}{4!},
\end{equation}%
$\Phi =\frac{\tau ^{\nicefrac{1}{2}}}{n^{\nicefrac{1}{2}}}$ is a Gaussian
integral (scaled by $n$), and the $n^{-2}$ remainder term may be
negative or positive. 
Condition \ref{eq:sigmadiff} easily follows from Equation~\ref{s} as
follows: take $\beta _{0},\beta _{1}$ in $\hat{B}^{\circ }$. By Equation~\ref%
{eq:geodesicform} we must have that $s(\beta _{0},n)-s(\beta _{1},n)=0$ for
all large $n$. But by Equation~\ref{s} this difference is equal to 
\begin{equation*}
cn^{{\text{-}}\nicefrac{3}{2}}\cdot (u(\beta _{0})-u(\beta _{1}))+O(n^{{%
\text{-}}2})
\end{equation*}%
for a constant $c>0$ independent of $\beta _{0}$ and $\beta _{1}$. 
Since this must be $0$ for all large $n$ and since $u(\cdot )$ does not
depend on $n$, this can only be true if $u(\beta _{0})=u(\beta _{1})$. Since
we can do this for any $\beta _{0}$ and $\beta _{1}$, Condition~\ref%
{eq:sigmadiff} follows. \newline
Now we proceed to prove the claim in Equation~\ref{s}. Define $A=[\beta
_{0}-c,\beta _{0}+c]$ for some fixed $c>0$, taken small enough so that $A$
is a subset of the interior of $B$ (this is why needed to restrict to $\hat{B%
}^{\circ }$ rather than $\hat{B}_{n}$). We can write 
\begin{equation}
s(\beta _{0},n)=f(\beta _{0},n)+g(\beta _{0},n)+h(\beta _{0},n)
\label{eq:poekie}
\end{equation}%
where we define: 
\begin{equation*}
f:=\int_{\beta \in A}\mathrm{e}^{{\text{-}}nD(\beta _{0}\Vert \beta )}~%
\mathrm{d}\beta ,
\end{equation*}%
\begin{equation*}
g:=\int_{\beta >\beta _{0}+c}\mathrm{e}^{{\text{-}}nD\left( \left. \beta
_{0}\right\Vert \beta \right) }~\mathrm{d}\beta \ \ \ \ h:=\int_{\beta
<\beta _{0}-c}\mathrm{e}^{{\text{-}}nD\left( \left. \beta _{0}\right\Vert
\beta \right) }~\mathrm{d}\beta 
\end{equation*}%
(We write $f$ instead of $f(\beta _{0},n)$ whenever $\beta _{0}$ and $n$ are
clear from context; similarly for $g,h$).

We have 
\begin{equation}
g\leq \sup_{\beta ^{\prime }>\beta _{0}+c}\mathrm{e}^{{\text{-}}(n-m)D\left(
\left. \beta _{0}\right\Vert \beta ^{\prime }\right) }\int_{\beta >\beta
_{0}+c}\mathrm{e}^{{\text{-}}mD\left( \left. \beta _{0}\right\Vert \beta
\right) }~\mathrm{d}\beta \leq c_{2}\mathrm{e}^{{\text{-}}c_{3}n^{c_{4}}}
\label{eq:expsmall}
\end{equation}%
for some constants $c_{2},c_{3},c_{4}>0$. Here we used that $D\left( \left.
\beta _{0}\right\Vert \beta ^{\prime }\right) $ is increasing in $\beta
^{\prime }$ so that the $\sup $ is achieved at $\beta _{0}+c$, and the fact
that by definition $m$ was chosen such that the integral with $mD\left(
\left. \beta _{0}\right\Vert \beta \right) $ in the exponent is finite. We
can bound $h$ similarly. Thus, the error we make if we neglect the integral
outside the set $A$ is negligible, and we can now concentrate on
approximating $f$, the integral over $A$. We can write 
\begin{equation}
f\left( \beta _{0},n\right) =\int_{A}\mathrm{e}^{{\text{-}}n\frac{1}{2}%
\left( \beta _{0}-\beta \right) ^{2}}\left( \mathrm{e}^{{\text{-}}n\frac{%
D_{3}}{3!}\left( \beta _{0}-\beta \right) ^{3}}\mathrm{e}^{{\text{-}}n\frac{%
D_{4}}{4!}\left( \beta _{0}-\beta \right) ^{4}}\mathrm{e}^{{\text{-}}n\cdot
O\left( \beta _{0}-\beta \right) ^{5}}\right) ~\mathrm{d}\beta 
\label{eq:hoekie}
\end{equation}%
where the constant in front of the 5th-order term is bounded because we
require $A$ to be a compact subset of the interior of $B$. The fourth- and
fifth-order terms in the integral can itself be well approximated by a
first-order Taylor approximation of $\mathrm{e}^{x}$ and we can rewrite $f$
as 
\begin{equation*}
\int_{A}\mathrm{e}^{{\text{-}}n\frac{1}{2}\left( \beta _{0}-\beta \right)
^{2}}\left( \mathrm{e}^{{\text{-}}n\frac{D_{3}}{3!}\left( \beta _{0}-\beta
\right) ^{3}}(1+V)(1+W)\right) ~\mathrm{d}\beta 
\end{equation*}%
where $V={\text{-}}n\frac{D_{4}}{4!}\left( \beta _{0}-\beta \right)
^{4}+O\left( n^{2}\left( \beta _{0}-\beta \right) ^{8}\right) $ and $%
W=O\left( n\left( \beta _{0}-\beta \right) ^{5}\right) $. Similarly, the
second factor in the integral can be well-approximated by a second order
Taylor approximation of $\mathrm{e}^{x}=1+x+(\nicefrac{1}{2})x^{2}+O(x^{3})$
so that we can further rewrite $f$ as 
\begin{equation*}
\begin{array}{c}
\int_{A}\mathrm{e}^{{\text{-}}n\frac{1}{2}\left( \beta _{0}-\beta \right)
^{2}}(1+U)(1+V)(1+W)~\mathrm{d}\beta = \\ 
\int_{A}\mathrm{e}^{{\text{-}}n\frac{1}{2}\left( \beta _{0}-\beta \right)
^{2}}(1+U+V+W+UV+UW+WV+UVW)~\mathrm{d}\beta 
\end{array}%
\end{equation*}%
where 
\begin{equation*}
U={\text{-}}n\frac{D_{3}}{3!}\left( \beta _{0}-\beta \right) ^{3}+\frac{1}{2}%
n^{2}\left( \frac{D_{3}}{3!}\right) ^{2}\left( \beta _{0}-\beta \right)
^{6}+O\left( n^{3}\left( \beta _{0}-\beta \right) ^{9}\right) .
\end{equation*}%
Writing $\Phi _{A}:=\int_{A}\mathrm{e}^{{\text{-}}n\frac{1}{2}\left( \beta
_{0}-\beta \right) ^{2}}~\mathrm{d}\beta $ we can thus further rewrite $f$
as 
\begin{equation*}
f=\Phi _{A}+\int_{A}\mathrm{e}^{{\text{-}}n\frac{1}{2}\left( \beta
_{0}-\beta \right) ^{2}}\left( U+V+R_{1}+R_{2}\right) ~\mathrm{d}\beta 
\end{equation*}%
where $R_{1}$ and $R_{2}$ are remainder terms, 
\begin{align*}
R_{1}=UV=\,& O\left( n^{2}\left\vert \beta _{0}-\beta \right\vert
^{7}\right) +O\left( n^{3}\left( \beta _{0}-\beta \right) ^{10}\right)  \\
& +O\left( n^{4}\left\vert \beta _{0}-\beta \right\vert ^{13}\right)
+O\left( n^{3}\left\vert \beta _{0}-\beta \right\vert ^{11}\right) +O\left(
n^{4}\left( \beta _{0}-\beta \right) ^{14}\right)  \\
& +O\left( n^{5}\left\vert \beta _{0}-\beta \right\vert ^{17}\right) 
\end{align*}%
and 
\begin{equation*}
R_{2}=W(1+U+V+UV))=O\left( n\left\vert \beta _{0}-\beta \right\vert
^{5}\right) .
\end{equation*}%
Since $\int_{-\infty }^{\infty }|x|^{m}\mathrm{e}^{{\text{-}}nx^{2}}~\mathrm{%
d}x=O\left( n^{({\text{-}}m-1)/2}\right) $, we have $\int_{A}\mathrm{e}^{{%
\text{-}}n\frac{1}{2}\left( \beta _{0}-\beta \right) ^{2}}\left(
R_{1}+R_{2}\right) ~\mathrm{d}\beta =O\left( n^{{\text{-}}2}\right) $, and
hence we get 
\begin{equation*}
f=\Phi _{A}+\int_{A}\mathrm{e}^{{\text{-}}n\frac{1}{2}\left( \beta
_{0}-\beta \right) ^{2}}\left( U+V\right) ~\mathrm{d}\beta +O\left(
n^{-2}\right) .
\end{equation*}%
Now, using the fact that $\int_{-a}^{a}x^{3}\mathrm{e}^{{\text{-}}nx^{2}}~%
\mathrm{d}x=0$ for all $a>0$, the integral over the first term in $U$ is $0$%
. The final terms in $U$ and $V$ can be dealt with as the remainder terms
above, and we can rewrite $f$ further as 
\begin{equation*}
f=\Phi _{A}+\int_{A}\mathrm{e}^{{\text{-}}n\frac{1}{2}\left( \beta
_{0}-\beta \right) ^{2}}\left( \frac{1}{2}n^{2}\left( \frac{D_{3}}{3!}%
\right) ^{2}\left( \beta _{0}-\beta \right) ^{6}-n\frac{D_{4}}{4!}\left(
\beta _{0}-\beta \right) ^{4}\right) ~\mathrm{d}\beta +O\left( n^{-2}\right)
.
\end{equation*}%
If we integrate over the full real line rather than $A$ then the error we
make is of order $O(\mathrm{e}^{{\text{-}}cn})\leq O(n^{{\text{-}}2})$. The
integrals over the real line can be evaluated whence we get: 
\begin{eqnarray}
f &=&\Phi +\frac{n^{2}}{2}\left( \frac{D_{3}}{3!}\right) ^{2}\cdot \left( 15%
\frac{\tau ^{\nicefrac{1}{2}}}{n^{\nicefrac{7}{2}}}\right) -n\frac{D_{4}}{4!}%
\cdot \left( 3\frac{\tau ^{\nicefrac{1}{2}}}{n}\right) +O(n^{-2})  \notag \\
&=&\Phi +n^{{\text{-}}\nicefrac{3}{2}}\cdot \tau ^{\nicefrac{1}{2}}\cdot
\left( 15\left( \frac{D_{3}}{3!}\right) ^{2}-6\frac{D_{4}}{4!}\right)
+O\left( n^{-2}\right) .
\end{eqnarray}%
Combining with (\ref{eq:poekie}) and (\ref{eq:expsmall}) that there exists a
constant, such that for all $n\geq m$, all $\beta _{0}\in \hat{B}^{\circ }$,%
\begin{equation}
5\left( D_{3}\right) ^{2}-3D_{4}=\mathrm{const}(\beta _{0}).
\label{eq:condstaerk}
\end{equation}%
We rephrase condition (\ref{eq:condstaerk}) in terms of the mean value
parameterization, we calculate higher derivatives of the divergence based on
(\ref{eq:andenafledt}) 
\begin{align*}
\frac{\partial ^{3}D\left( \left. \beta _{0}\right\Vert \beta \right) }{%
\partial \beta ^{3}}& ={\text{-}}\frac{\mathrm{d}\sigma }{\mathrm{d}\mu }%
+\left( \mu -\mu _{0}\right) \cdot \left( \sigma ^{{\text{-}}1}\left( \frac{%
\mathrm{d}\sigma }{\mathrm{d}\mu }\right) ^{2}-\frac{\mathrm{d}^{2}\sigma }{%
\mathrm{d}\mu ^{2}}\right) , \\
\frac{\partial ^{4}D\left( \left. \beta _{0}\right\Vert \beta \right) }{%
\partial \beta ^{4}}& =\left( \frac{\mathrm{d}\sigma }{\mathrm{d}\mu }%
\right) ^{2}-2\sigma \frac{\mathrm{d}^{2}\sigma }{\mathrm{d}\mu ^{2}}+\left(
\mu -\mu _{0}\right) \cdot \left( {\text{-}}\sigma ^{{\text{-}}1}\left( 
\frac{\mathrm{d}\sigma }{\mathrm{d}\mu }\right) ^{3}+2\frac{\mathrm{d}\sigma 
}{\mathrm{d}\mu }\frac{\mathrm{d}^{2}\sigma }{\mathrm{d}\mu ^{2}}-\sigma 
\frac{\mathrm{d}^{3}\sigma }{\mathrm{d}\mu ^{3}}\right) , \\
\frac{\partial ^{5}D\left( \left. \beta _{0}\right\Vert \beta \right) }{%
\partial \beta ^{5}}& ={\text{-}}\left( \frac{\mathrm{d}\sigma }{\mathrm{d}%
\mu }\right) ^{3}+2\sigma \frac{\mathrm{d}\sigma }{\mathrm{d}\mu }\frac{%
\mathrm{d}^{2}\sigma }{\mathrm{d}\mu ^{2}}-3\sigma ^{2}\frac{\mathrm{d}%
^{3}\sigma }{\mathrm{d}\mu ^{3}} \\
+& \left( \mu -\mu _{0}\right) \cdot \left( \sigma ^{{\text{-}}1}\left( 
\frac{\mathrm{d}\sigma }{\mathrm{d}\mu }\right) ^{4}-3\left( \frac{\mathrm{d}%
\sigma }{\mathrm{d}\mu }\right) ^{2}\frac{\mathrm{d}^{2}\sigma }{\mathrm{d}%
\mu ^{2}}+2\sigma \left( \frac{\mathrm{d}^{2}\sigma }{\mathrm{d}\mu ^{2}}%
\right) ^{2}+\sigma \frac{\mathrm{d}\sigma }{\mathrm{d}\mu }\frac{\mathrm{d}%
^{3}\sigma }{\mathrm{d}\mu ^{3}}-\sigma ^{2}\frac{\mathrm{d}^{4}\sigma }{%
\mathrm{d}\mu ^{4}}\right) , \\
\frac{\partial ^{6}D\left( \left. \beta _{0}\right\Vert \beta \right) }{%
\partial \beta ^{6}}& =\left( \frac{\mathrm{d}\sigma }{\mathrm{d}\mu }%
\right) ^{4}-4\sigma \left( \frac{\mathrm{d}\sigma }{\mathrm{d}\mu }\right)
^{2}\frac{\mathrm{d}^{2}\sigma }{\mathrm{d}\mu ^{2}}-3\sigma ^{2}\frac{%
\mathrm{d}\sigma }{\mathrm{d}\mu }\frac{\mathrm{d}^{3}\sigma }{\mathrm{d}\mu
^{3}}+4\sigma ^{2}\left( \frac{\mathrm{d}^{2}\sigma }{\mathrm{d}\mu ^{2}}%
\right) ^{2}-4\sigma ^{3}\frac{\mathrm{d}^{4}\sigma }{\mathrm{d}\mu ^{4}} \\
+& \left( \mu -\mu _{0}\right) \cdot \left( \dots \right) .
\end{align*}

Then 
\begin{align*}
D_{3}& ={\text{-}}\frac{\mathrm{d}\sigma }{\mathrm{d}\mu }, \\
D_{4}& =\left( \frac{\mathrm{d}\sigma }{\mathrm{d}\mu }\right) ^{2}-2\sigma 
\frac{\mathrm{d}^{2}\sigma }{\mathrm{d}\mu ^{2}}, \\
D_{5}& ={\text{-}}\left( \frac{\mathrm{d}\sigma }{\mathrm{d}\mu }\right)
^{3}+2\sigma \frac{\mathrm{d}\sigma }{\mathrm{d}\mu }\frac{\mathrm{d}%
^{2}\sigma }{\mathrm{d}\mu ^{2}}-3\sigma ^{2}\frac{\mathrm{d}^{3}\sigma }{%
\mathrm{d}\mu ^{3}}, \\
D_{6}& =\left( \frac{\mathrm{d}\sigma }{\mathrm{d}\mu }\right) ^{4}-4\sigma
\left( \frac{\mathrm{d}\sigma }{\mathrm{d}\mu }\right) ^{2}\frac{\mathrm{d}%
^{2}\sigma }{\mathrm{d}\mu ^{2}}-3\sigma ^{2}\frac{\mathrm{d}\sigma }{%
\mathrm{d}\mu }\frac{\mathrm{d}^{3}\sigma }{\mathrm{d}\mu ^{3}}+4\sigma
^{2}\left( \frac{\mathrm{d}^{2}\sigma }{\mathrm{d}\mu ^{2}}\right)
^{2}-4\sigma ^{3}\frac{\mathrm{d}^{4}\sigma }{\mathrm{d}\mu ^{4}},
\end{align*}%
where, as before, $\mu _{0}=\mu \left( \beta _{0}\right) $. Further we get 
\begin{eqnarray*}
5\left( D_{3}\right) ^{2}-3D_{4} &=&5\left( -\frac{\mathrm{d}\sigma }{%
\mathrm{d}\mu }\right) ^{2}-3\left( \left( \frac{\mathrm{d}\sigma }{\mathrm{d%
}\mu }\right) ^{2}-2\sigma \frac{\mathrm{d}^{2}\sigma }{\mathrm{d}\mu ^{2}}%
\right) \\
&=&2\left( \frac{\mathrm{d}\sigma }{\mathrm{d}\mu }\right) ^{2}+6\sigma \frac{%
\mathrm{d}^{2}\sigma }{\mathrm{d}\mu ^{2}}.
\end{eqnarray*}

Plugging the above into (\ref{eq:condstaerk}) and rearranging the terms
gives the following differential equation for $\sigma $%
\begin{equation}
\left( \frac{\mathrm{d}\sigma }{\mathrm{d}\mu }\right) ^{2}+3\sigma \frac{%
\mathrm{d}^{2}\sigma }{\mathrm{d}\mu ^{2}}=\mathrm{const}(\mu ).
\label{eq:foerstedif}
\end{equation}%
This is a necessary condition for exchangability.
\end{proof}

\begin{proof}[Theorem~\protect\ref{sufficiency}]

We also need to take a closer look at higher-order terms in the Taylor
expansion of the integral (\ref{eq:geodesicform}) and obtain a stronger
necessary condition for exchangeability. As in the proof Theorem \ref%
{disjunction of conditions}, we expand the integral over $A=\left[ \beta
_{0}-c,\beta _{0}+c\right] $: 
\begin{align*}
f\left( \beta _{0},n\right) & =\int_{A}\mathrm{e}^{{\text{-}}n\frac{1}{2}%
\left( \beta _{0}-\beta \right) ^{2}}\left( \prod_{k=3}^{6}\mathrm{e}^{{%
\text{-}}n\frac{D_{k}}{k!}\left( \beta _{0}-\beta \right) ^{k}}\right) 
\mathrm{e}^{{\text{-}}nO\left( \left( \beta _{0}-\beta \right) ^{7}\right) }~%
\mathrm{d}\beta  \\
& =\int_{A}\mathrm{e}^{{\text{-}}n\frac{1}{2}\left( \beta _{0}-\beta \right)
^{2}}\left( \prod_{k=3}^{7}\left( 1+X_{k}\right) \right) ~\mathrm{d}\beta ~,
\end{align*}%
where 
\begin{align*}
X_{3}& ={\text{-}}n\frac{D_{3}}{3!}\left( \beta _{0}-\beta \right) ^{3}+%
\frac{1}{2}n^{2}\left( \frac{D_{3}}{3!}\right) ^{2}\left( \beta _{0}-\beta
\right) ^{6}+\frac{1}{3!}n^{3}\left( \frac{D_{3}}{3!}\right) ^{3}\left(
\beta _{0}-\beta \right) ^{9} \\
& +\frac{1}{4!}n^{4}\left( \frac{D_{3}}{3!}\right) ^{4}\left( \beta
_{0}-\beta \right) ^{12}+O\left( n^{5}(\beta _{0}-\beta )^{15}\right) ~, \\
X_{4}& ={\text{-}}n\frac{D_{4}}{4!}\left( \beta _{0}-\beta \right) ^{4}+%
\frac{1}{2}n^{2}\left( \frac{D_{4}}{4!}\right) ^{2}\left( \beta _{0}-\beta
\right) ^{8}+O\left( n^{3}\left( \beta _{0}-\beta \right) ^{12}\right) \,, \\
X_{5}& ={\text{-}}n\frac{D_{5}}{5!}\left( \beta _{0}-\beta \right)
^{5}+O\left( n^{2}\left( \beta _{0}-\beta \right) ^{10}\right) ~, \\
X_{6}& ={\text{-}}n\frac{D_{6}}{6!}\left( \beta _{0}-\beta \right)
^{6}+O\left( n^{2}\left( \beta _{0}-\beta \right) ^{12}\right) ~, \\
X_{7}& ={\text{-}}O\left( n\left( \beta _{0}-\beta \right) ^{7}\right) ~.
\end{align*}%
We assume that condition (\ref{eq:condstaerk}) is satisfied, so that $%
O\left( n^{\textrm{-}\nicefrac{3}{2}}\right) $ term in the expansion (cf. Equation %
\ref{s}) is constant in $\beta _{0}$. Since if we integrate over the full
real line rather than $A$ then the error we make is of order $O\left( 
\mathrm{e}^{\textrm{-}cn}\right) $, and $\left( \beta _{0}-\beta \right) ^{m}$ under
Gaussian integral over the full real line results in $O\left( n^{{\text{-}}%
(m+1)/2}\right) $ if $m$ is even, and $0$ if $m$ is odd, there will be no
terms of order $O\left( n^{{\text{-}}2}\right) $. Therefore, we need to look
for terms of order $O\left( n^{{\text{-}}\nicefrac{5}{2}}\right) $. There
are five of them and their sum must be independent of $\beta _{0}$ (using
similar argument as for the $O\left( n^{{\text{-}}\nicefrac{3}{2}}\right) $
term in the proof of Theorem \ref{disjunction of conditions}): 
\begin{align*}
& \frac{1}{4!}n^{4}\left( \frac{D_{3}}{3!}\right) ^{4}\left( \beta
_{0}-\beta \right) ^{12}~+~\frac{1}{2}n^{2}\left( \frac{D_{4}}{4!}\right)
^{2}\left( \beta _{0}-\beta \right) ^{8}~-~n\frac{D_{6}}{6!}\left( \beta
_{0}-\beta \right) ^{6} \\
& ~-~\frac{1}{2}n^{3}\left( \frac{D_{3}}{3!}\right) ^{2}\frac{D_{4}}{4!}%
\left( \beta _{0}-\beta \right) ^{10}~+~n^{2}\frac{D_{3}}{3!}\frac{D_{5}}{5!}%
\left( \beta _{0}-\beta \right) ^{8}=\mathrm{const}\left( \beta _{0}\right)
~.
\end{align*}%
All the terms appear in the Gaussian integral. Given the fact that for even $%
m$, 
\begin{equation*}
\int \mathrm{e}^{{\text{-}}n\frac{1}{2}(\beta _{0}-\beta )^{2}}\left( \beta
_{0}-\beta \right) ^{m}~\mathrm{d}\beta =(m-1)!!\tau ^{\nicefrac{1}{2}}n^{{%
\text{-}}\frac{m+1}{2}},
\end{equation*}%
we can rewrite the condition on $O\left( n^{\textrm{-}\nicefrac{5}{2}}\right) $ term
as: 
\begin{equation*}
\frac{11!!}{4!(3!)^{4}}D_{3}^{4}+\frac{7!!}{2(4!)^{2}}D_{4}^{2}-\frac{5!!}{6!%
}D_{6}-\frac{9!!}{2(3!)^{2}4!}D_{3}^{2}D_{4}+\frac{7!!}{3!5!}D_{3}D_{5}=%
\mathrm{const}\left( \beta _{0}\right) ,
\end{equation*}%
where we also skipped the $n^{\textrm{-}\nicefrac{5}{2}}$ terms and used the fact
that $D_{2}=1$ in the geodesic parameterization. Evaluating the factorials
and multiplying by a constant gives: 
\begin{equation}
385D_{3}^{4}+105D_{4}^{2}-24D_{6}-630D_{3}^{2}D_{4}+168D_{3}D_{5}=\mathrm{%
const}(\beta _{0}).  \label{eq:5_2}
\end{equation}%
In order to evaluate $D_{3},D_{4},D_{5},$ and $D_{6}$ we calculate the
derivatives of $\sigma $ under the condition that the differential equation (%
\ref{eq:foerstedif}) is satisfied for some constant $c$%
\begin{eqnarray*}
2\left( \frac{\mathrm{d}\sigma }{\mathrm{d}\mu }\right) ^{2}+6\sigma \frac{%
\mathrm{d}^{2}\sigma }{\mathrm{d}\mu ^{2}} &=&3c, \\
\frac{\mathrm{d}^{2}\sigma }{\mathrm{d}\mu ^{2}} &=&\frac{1}{2}\sigma ^{{%
\text{-}}1}c-\frac{1}{3}\sigma ^{{\text{-}}1}\left( \frac{\mathrm{d}\sigma }{%
\mathrm{d}\mu }\right) ^{2}, \\
\frac{\mathrm{d}^{3}\sigma }{\mathrm{d}\mu ^{3}} &=&{\text{-}}\frac{5}{6}%
\sigma ^{{\text{-}}2}\frac{\mathrm{d}\sigma }{\mathrm{d}\mu }c+\frac{5}{9}%
\sigma ^{{\text{-}}2}\left( \frac{\mathrm{d}\sigma }{\mathrm{d}\mu }\right)
^{3}, \\
\frac{\mathrm{d}^{4}\sigma }{\mathrm{d}\mu ^{4}} &=&\frac{25}{9}\sigma ^{{%
\text{-}}3}c\left( \frac{\mathrm{d}\sigma }{\mathrm{d}\mu }\right) ^{2}-%
\frac{5}{3}\sigma ^{{\text{-}}3}\left( \frac{\mathrm{d}\sigma }{\mathrm{d}%
\mu }\right) ^{4}-\frac{5}{12}\sigma ^{{\text{-}}3}c^{2}.
\end{eqnarray*}%
We plug this into (\ref{eq:5_2}), and get:%
\begin{align*}
D_{2}& =1, \\
D_{3}& ={\text{-}}\frac{\mathrm{d}\sigma }{\mathrm{d}\mu }, \\
D_{4}& =\frac{5}{3}\left( \frac{\mathrm{d}\sigma }{\mathrm{d}\mu }\right)
^{2}-c, \\
D_{5}& ={\text{-}}\frac{10}{3}\left( \frac{\mathrm{d}\sigma }{\mathrm{d}\mu }%
\right) ^{3}+\frac{7c}{2}\frac{\mathrm{d}\sigma }{\mathrm{d}\mu }, \\
D_{6}& =\frac{2\cdot 35}{9}\left( \frac{\mathrm{d}\sigma }{\mathrm{d}\mu }%
\right) ^{4}-\frac{215c}{2\cdot 9}\left( \frac{\mathrm{d}\sigma }{\mathrm{d}%
\mu }\right) ^{2}+\frac{8}{3}c^{2}.
\end{align*}

We plug this into (\ref{eq:5_2}), and get:%
\begin{eqnarray*}
&&385D_{3}^{4}+105D_{4}^{2}-24D_{6}-630D_{3}^{2}D_{4}+168D_{3}D_{5} \\
&=&385\left( -\frac{\mathrm{d}\sigma }{\mathrm{d}\mu }\right) ^{4}+105\left( 
\frac{5}{3}\left( \frac{\mathrm{d}\sigma }{\mathrm{d}\mu }\right)
^{2}-c\right) ^{2}-24\left( \frac{2\cdot 35}{9}\left( \frac{\mathrm{d}\sigma 
}{\mathrm{d}\mu }\right) ^{4}-\frac{215c}{2\cdot 9}\left( \frac{\mathrm{d}%
\sigma }{\mathrm{d}\mu }\right) ^{2}+\frac{8}{3}c^{2}\right) \\
&&-630\left( {\text{-}}\frac{\mathrm{d}\sigma }{\mathrm{d}\mu }\right)
^{2}\left( \frac{5}{3}\left( \frac{\mathrm{d}\sigma }{\mathrm{d}\mu }\right)
^{2}-c\right) +168\left( {\text{-}}\frac{\mathrm{d}\sigma }{\mathrm{d}\mu }%
\right) \left( {\text{-}}\frac{10}{3}\left( \frac{\mathrm{d}\sigma }{\mathrm{%
d}\mu }\right) ^{3}+\frac{7c}{2}\frac{\mathrm{d}\sigma }{\mathrm{d}\mu }%
\right) \\
&=&385\left( \frac{\mathrm{d}\sigma }{\mathrm{d}\mu }\right) ^{4}+105\cdot
\left( \frac{5}{3}\right) ^{2}\left( \frac{\mathrm{d}\sigma }{\mathrm{d}\mu }%
\right) ^{4}+105\cdot c^{2}-35\cdot 2\cdot 5\left( \frac{\mathrm{d}\sigma }{%
\mathrm{d}\mu }\right) ^{2}c-\frac{8\cdot 2\cdot 35}{3}\left( \frac{\mathrm{d%
}\sigma }{\mathrm{d}\mu }\right) ^{4} \\
&&+\frac{4\cdot 215c}{3}\left( \frac{\mathrm{d}\sigma }{\mathrm{d}\mu }%
\right) ^{2}-8^{2}c^{2} \\
&&-210\cdot 5\left( \frac{\mathrm{d}\sigma }{\mathrm{d}\mu }\right)
^{4}+630\left( \frac{\mathrm{d}\sigma }{\mathrm{d}\mu }\right) ^{2}c+56\cdot
10\left( \frac{\mathrm{d}\sigma }{\mathrm{d}\mu }\right) ^{4}-84\cdot
7c\left( \frac{\mathrm{d}\sigma }{\mathrm{d}\mu }\right) ^{2}.
\end{eqnarray*}

Collecting the terms gives%
\begin{eqnarray*}
&&\left( 385+105\cdot \left( \frac{5}{3}\right) ^{2}-\frac{8\cdot 2\cdot 35}{%
3}-210\cdot 5+56\cdot 10\right) \left( \frac{\mathrm{d}\sigma }{\mathrm{d}%
\mu }\right) ^{4} \\
&&-\left( 35\cdot 2\cdot 5-\frac{4\cdot 215}{3}-630+84\cdot 7\right) c\left( 
\frac{\mathrm{d}\sigma }{\mathrm{d}\mu }\right) ^{2}+\left( 105-8^{2}\right)
c^{2} \\
&=&{\text{-}}\frac{64}{3}c\left( \frac{\mathrm{d}\sigma }{\mathrm{d}\mu }%
\right) ^{2}+41c^{2}
\end{eqnarray*}%
Interestingly all term with $\left( \mathrm{d}\sigma /\mathrm{d}\mu \right)
^{4}$ have disappeared and we get:%
\begin{equation}
{\text{-}}\frac{64}{3}c\left( \frac{\mathrm{d}\sigma }{\mathrm{d}\mu }%
\right) ^{2}+41c^{2}=\mathrm{const}(\mu ).  \label{eq:final_equation}
\end{equation}%
\textbf{Assume} that $c\neq 0$. Equation (\ref{eq:final_equation}) is
satisfied only when 
\begin{equation*}
\frac{\mathrm{d}\sigma }{\mathrm{d}\mu }=\mathrm{const}(\mu ),
\end{equation*}%
which has a general solution of the form:%
\begin{equation*}
\sigma \left( \mu \right) =k\mu +\ell
\end{equation*}%
for some constants $c_{1}$ and $k$ and we get%
\begin{equation*}
V\left( \mu \right) =\left( k\mu +\ell \right) ^{2}.
\end{equation*}

\textbf{Assume }$c=0.$ We now solve (\ref{eq:foerstedif}). The differential
equation can be rewritten as%
\begin{equation*}
\frac{\mathrm{d}^{2}}{\mathrm{d}\mu ^{2}}\big({\sigma }^{\nicefrac{4}{3}}\big)=0.
\end{equation*}%
Hence there exists constants $k$ and $\ell $ such that%
\begin{equation*}
\sigma ^{\nicefrac{4}{3}}=k\mu +\ell
\end{equation*}%
or equivalently%
\begin{equation*}
V\left( \mu \right) =\left( k\mu +\ell \right) ^{\nicefrac{3}{2}}.
\end{equation*}
\end{proof}

\end{document}